\relax
\documentclass[letterpaper]{article} 
\usepackage{aaai21}  
\usepackage{times}  
\usepackage{helvet} 
\usepackage{courier}  
\usepackage[hyphens]{url}  
\usepackage{graphicx} 
\usepackage{graphicx}  
\usepackage{natbib}  
\usepackage{caption} 
\usepackage{amsmath}
\usepackage{amsthm}
\usepackage{amsfonts}
\usepackage{algorithm}
\usepackage{algorithmic}
\usepackage{subfigure}
\usepackage{epstopdf}
\usepackage{enumitem}
\usepackage{multirow}
\usepackage{diagbox}
\usepackage{eqparbox}

\usepackage{dblfloatfix}
\newtheorem{lemma}{Lemma}
\newtheorem{theorem}{Theorem}
\newtheorem{proposition}[theorem]{Proposition}

\usepackage[switch]{lineno}

\title{Effective and Sparse Count-Sketch via $k$-means clustering}
\author {
    Yuhan Wang,
    Zijian Lei,
    Liang Lan
  \\
}
\affiliations {
     Department of Computer Science \\
     Hong Kong Baptist University \\
     Hong Kong SAR, China\\
    19437595@life.hkbu.edu.hk, \{cszjlei,lanliang\}@comp.hkbu.edu.hk
}
\begin{document}
\maketitle

\begin{abstract}
Count-sketch is a popular matrix sketching algorithm that can produce a sketch of an input data matrix $\mathbf{X}$ in $O(nnz(\mathbf{X}))$ time where $nnz(\mathbf{X})$ denotes the number of non-zero entries in $\mathbf{X}$. The sketched matrix will be much smaller than $\mathbf{X}$ while preserving most of its properties. Therefore, count-sketch is widely used for addressing high-dimensionality challenge in machine learning. However, there are two main limitations of count-sketch: (1) The sketching matrix used count-sketch is generated randomly which does not consider any intrinsic data properties of $\mathbf{X}$. This data-oblivious matrix sketching method could produce a bad sketched matrix which will result in low accuracy for subsequent machine learning tasks (e.g., classification); (2) For highly sparse input data, count-sketch could produce a dense sketched data matrix. This dense sketch matrix could make the subsequent machine learning tasks more computationally expensive than on the original sparse data $\mathbf{X}$. To address these two limitations, we first show an interesting connection between count-sketch and $k$-means clustering by analyzing the reconstruction error of the count-sketch method. Based on our analysis, we propose to reduce the reconstruction error of count-sketch by using $k$-means clustering algorithm to obtain the low-dimensional sketched matrix. In addition, we propose to solve $k$-mean clustering using gradient descent with $\epsilon$-$\mathcal{L}_1$ ball projection to produce a sparse sketched matrix. Our experimental results based on six real-life classification datasets have demonstrated that our proposed method achieves higher accuracy than the original count-sketch and other popular matrix sketching algorithms. Our results also demonstrate that our method produces a sparser sketched data matrix than other methods and therefore the prediction cost of our method will be smaller than other matrix sketching methods. 
\end{abstract}

\section{Introduction}
Matrix sketching \cite{woodruff2014sketching} is a powerful dimensionality reduction method that can efficiently find a small matrix to replace the original large matrix while preserving most of its properties. For an input large data matrix $\mathbf{X} \in \mathbb{R}^{n \times d}$ where $n$ is the number of samples and $d$ is the number of features, matrix sketching methods generate a sketch of $\mathbf{X}$ by multiplying it with a random sketching matrix $\mathbf{R} \in \mathbb{R}^{d \times r}$ $(r \ll d)$ with certain properties. Compared with traditional dimensionality reduction methods (e.g., Principal Component Analysis (PCA) \cite{jolliffe2011principal}), matrix sketching methods can obtain the sketched matrix very efficient with certain theoretical guarantees \cite{woodruff2014sketching}. Therefore, matrix sketching has gained significant research attention and has been used widely for handling high-dimensional data in machine learning \cite{mahoney2011randomized, ailon2006approximate, bojarski2017structured, choromanski2017unreasonable}.

A typical way of applying matrix sketching in machine learning problem is \textit{sketch and solve} \cite{dahiya2018empirical}. For example, in a linear classification problem with training data $\{\mathbf{X}, \mathbf{y}\}$ where $\mathbf{X} \in \mathbb{R}^{n \times d}$ is a large input feature matrix and $\mathbf{y} \in \mathbb{R}^{n}$ is the corresponding label vector, a classification model can be trained by solving $\min_{\mathbf{w} \in \mathbb{R}^d }\sum_{i=1}^{n}l(\mathbf{w}^T\mathbf{x}_i, y_i) + \lambda\|\mathbf{w}\|_2$ where $l(\cdot)$ denotes a loss function (e.g., hinge loss). By using matrix sketching, we can first obtain a sketched data matrix $\widetilde{\mathbf{X}} \in \mathbb{R}^{n \times r}$ by $\widetilde{\mathbf{X}} = \mathbf{XR}$ and then solve a much smaller problem $\min_{\mathbf{v} \in \mathbb{R}^r}\sum_{i=1}^{n}l(\mathbf{v}^T\widetilde{\mathbf{x}}_i, y_i) + \lambda\|\mathbf{v}\|_2$. Then, the expensive computation on original large matrix $\mathbf{X}$ can be replaced by computation on small matrix $\widetilde{\mathbf{X}}$. This \textit{sketch and solve} method has also been used to speedup other machine learning tasks, such as least squares regression \cite{dobriban2019asymptotics}, low-rank approximation \cite{tropp2017practical,clarkson2017low} and $k$-means clustering \cite{boutsidis2010random, liu2017sparse}.

Recent advances in randomized numerical linear algebra \cite{martinsson2020randomized} has provided a solid theoretical foundation for matrix sketching. Various methods have been proposed to construct the random matrix $\mathbf{R}$. The early method \cite{dasgupta1999elementary} constructs a dense random Gaussian matrix $\mathbf{R}$ where each element in $\mathbf{R}$ is generated from a Gaussian distribution $\mathcal{N}(0, \frac{1}{d})$. This method based on dense random Gaussian matrix $\mathbf{R}$ requires $O(ndr)$ time for computing the sketched matrix $\widetilde{\mathbf{X}} = \mathbf{XR}$. \citet{achlioptas2003database} proposed to generate a sparser random matrix $\mathbf{R}$ where each element in $\mathbf{R}$ is generated from \{$-1$, 0, 1\} following a discrete distribution. It will reduce the computation complexity from $O(ndr)$ to $O(\frac{1}{3}ndr)$. In recent years, two famous fast random projection matrices were proposed for efficiently computing the projection $\mathbf{XR}$. The first one is the Subsampled Randomized Hadamard Transform (SRHT) which can achieve $O(ndlog(r))$ time for computing $\mathbf{XR}$ \cite{tropp2011improved, ailon2009fast}. The second method is called count-sketch \cite{clarkson2017low} which can compute $\mathbf{XR}$ in $O(nnz(\mathbf{X}))$ time for any input $\mathbf{X}$ which makes the count-Sketch method particularly suitable for sparse input data. In our paper, we focus on improving the count-sketch algorithm in the context of classification. 

Count-sketch constructs the random matrix $\mathbf{R}$ by a product of two matrices $\mathbf{D}$ and $\mathbf{\Phi}$, i.e., $\mathbf{R = D\Phi}$, where $\mathbf{D} \in \mathbb{R}^{d\times d}$ is a random diagonal matrix where each diagonal values is uniformly chosen from $\{1, -1\}$ and $\mathbf{\Phi} \in \mathbb{R}^{d\times r}$ is a very sparse matrix where each row has only one randomly selected entry equal to 1 and all other are 0. Previously \citet{paul2014random} applied count-sketch for linear SVM classification and showed that linear SVM trained on the sketched data matrix can ensure comparable generalization ability as in the original space in the case of classification. However, there are two main limitations of count-sketch: (1) It is a data-oblivious method where the generation of sketching matrix $\mathbf{R}$ is totally independent of input data matrix $\mathbf{X}$ and therefore the sketched matrix may not be effective for the subsequent classification algorithm; (2) The sketched data matrix $\Tilde{\mathbf{X}}$ will not maintain the same sparsity rate as the original input data $\mathbf{X}$. It could make the subsequent classification algorithm on the sketched data more computationally expensive than on the original data $\mathbf{X}$. Even though data-oblivious matrix sketching has been extensively studied, few studies focus on efficient data-dependent matrix sketching. Recently, \citet{xu2017efficient} proposed to use the approximated singular value decomposition (SVD) as the projection subspace. \citet{DBLP:conf/aaai/LeiL20} proposed to improve SRHT by non-uniform sampling by exploiting data properties. However, both of them will produce a dense sketched matrix for sparse input data.    

In this paper, we focus on addressing the aforementioned two limitations of count-sketch. To address the first limitation, we first show an interesting connection between count-sketch and $k$-means clustering by analyzing the reconstruction error of count-sketch. Based on our analysis, we propose to reduce the reconstruction error of count-sketch by using $k$-means clustering to obtain the low-dimensional sketched data matrix. To address the second limitation, we propose to get sparse cluster centers by optimizing $k$-means objective function using gradient descent with $\epsilon$-$\mathcal{L}_1$ ball projection in each iteration. Finally, we compare our proposed methods with the other five popular matrix sketching algorithms on six real-life datasets. Our experimental results clearly demonstrate that our proposed data-dependent matrix sketching methods achieve higher accuracy than count-sketch and other random matrix sketching algorithms. Our results also show our method produces a sparser sketched data matrix than count-sketch and other matrix sketching methods. The prediction cost of our method is smaller than other matrix sketching methods.

\section{Preliminaries}
\subsection{Randomized Matrix Sketching}
Given a data matrix $\mathbf{X} \in \mathbb{R}^{n \times d}$ and a random sketching matrix $\mathbf{R} \in \mathbb{R}^{d \times r}$ with $r \ll d$, a sketched matrix is produced by 
\begin{eqnarray}
\mathbf{\widetilde{X}} = \mathbf{XR} \in \mathbb{R}^{n \times r}.
\end{eqnarray}

Note that the matrix $\mathbf{R}$ is randomly generated and is independent of the input data $\mathbf{X}$. As shown in the following Johnson-Lindenstrauss Lemma (JL lemma), randomized matrix sketching can preserve the pairwise distance of all data points using the sketched data matrix $\mathbf{\widetilde{X}}$. 

\begin{lemma} [Johnson-Lindenstrauss Lemma (JL lemma) \cite{johnson1984extensions}]\label{theorem:JL_lemma}
For any $0<\epsilon<1$ and any integer n, let $r = O(\log n/\epsilon^2)$ and $\mathbf{R} \in \mathbb{R}^{d \times r}$ be a random orthonormal matrix. Then for any set $\mathbf{X}$ of $n$ points in $\mathbb{R}^d$, the following inequality about pairwise distance between any two data points $\mathbf{x}_i$ and $\mathbf{x}_j$ in $\mathbf{X}$ holds true with high probability:
\begin{equation*}
(1-\epsilon)\|\mathbf{x}_i-\mathbf{x}_j\|_2 \le \|\mathbf{R}^{T}\mathbf{x}_i-\mathbf{R}^{T}\mathbf{x}_j\|_2 \le(1+\epsilon)\|\mathbf{x}_i-\mathbf{x}_j\|_2.
\end{equation*}
\end{lemma}
\subsection{Count-Sketch}
Among various methods for constructing the sketching matrix $\mathbf{R}$, count-sketch (or called sparse embedding) is well suited for sparse input data $\mathbf{X}$ since it can achieve $O(nnz(\mathbf{X}))$ time complexity for computing $\mathbf{XR}$. 
Count-sketch \cite{clarkson2017low} constructs the random matrix $\mathbf{R} \in \mathbb{R}^{d \times r}$ as $\mathbf{R = D\Phi}$ where $\mathbf{D}$ and $\mathbf{\Phi}$ are defined as follows,
\begin{itemize}
    \item $\mathbf{D}$ is a $d \times d$ diagonal matrix with each diagonal entry independently chosen to be 1 or $-1$ with probability 0.5.
    \item $\mathbf{\Phi}\in \{0,1\}^{d\times r}$ is a $d \times r$ binary matrix with $\mathbf{\Phi}_{i, h(i)}$ = 1, and all remaining entries 0. $h$ is a random map such that for any $i \in \{1, 2, \dots, d\}$, $h(i) = j$, for $j \in \{1, 2, \dots, r\}$ with probability $\frac{1}{r}$.
\end{itemize}
Note that random sketching matrix $\mathbf{R}$ in count-sketch is a very sparse matrix where each row have only one nonzero entry. This nonzero entry is uniformly chosen and the value is either 1 or $-1$ with probability 0.5. $\mathbf{XR}$ can be computed in $O(nnz(\mathbf{X}))$ time because each nonzero entry in $\mathbf{X}$ is at most by multiplied by one nonzero entry in $\mathbf{XR}$.  

\section{Methodology}
Even though count-sketch has been successfully used for dimensionality reduction in linear SVM classification \citet{paul2014random}, we argue that this data-oblivious method has two limitations: (1) The sketching matrix $\mathbf{R = D\Phi}$ is randomly generated. It could result in bad sketched data when some important columns in $\mathbf{X}$ are not sampled by using $\mathbf{R}$; (2) Count-sketch will not preserve the sparsity rate of the original data. 

When applying count-sketch for data classification, the first limitation could result in bad low-dimensional embedding and then produce a classification model with low accuracy. To illustrate this limitation, we show the classification accuracy of using count-sketch for dimensionality reduction on mnist dataset for ten different runs in Figure \ref{fig1}. As shown in Figure \ref{fig1}, count-sketch (the blue line with triangle markers) could produce low classification accuracy in some runs and also the accuracy is not stable. We also show the classification of our proposed method that will be introduced later in this figure (the red line with circle markers). We can see that our proposed method produces significantly better accuracy than count-sketch.

\begin{figure}[!t]
\centering
\includegraphics[width=0.9\columnwidth]{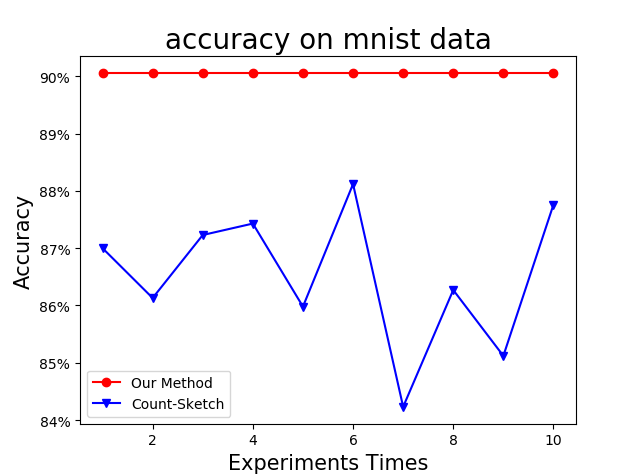}
\caption{Classification Accuracy of Using Count-sketch on Different Runs}

\label{fig1}
\end{figure}

The second limitation of count-sketch is that, when used with sparse input data, the sketched matrix could be much denser than the original data. We checked the sparsity rate of mnist data before and after count-sketch. The original sparsity rate for mnist data is 80.78\% and the sparsity rate is significantly decreased to 1.72\% in the sketched data. Therefore, the sketched data could contain more nonzero values than the original data and make the subsequent classification algorithm slower. More examples can be found in the experiment section.     

\subsection{Connection between Count-Sketch and $k$-means clustering}
Since the construction of matrix $\mathbf{D}$ and $\mathbf{\Phi}$ in count-sketch is oblivious to the input data matrix $\mathbf{X}$, it could produce a bad sketched matrix (e.g., some important columns in $\mathbf{X}$ are not be sampled in $\mathbf{\Phi}$) and therefore results in low classification accuracy. In this paper, we seek to develop a data-dependent count-sketch method for addressing the two limitations of count-sketch. To motivate our method, we start by analyzing the reconstruction error of the original count-sketch method and show an interesting connection between count-sketch and $k$-means clustering. 

Let us define a diagonal scaling matrix $\mathbf{S} \in \mathbb{R}^{r \times r} $ as 
\begin{equation}
 \mathbf{S}_{ii} =  \frac{1}{\sum_{j=1}^{d}\Phi_{ji}}.
\end{equation}
Note that $(\mathbf{D\Phi}\mathbf{ S}^{\frac{1}{2}})^T(\mathbf{D\Phi}\mathbf{ S}^{\frac{1}{2}})$ equals to an identity matrix with size $r \times r$. 
The reconstruction error of count-sketch can be represented as
\begin{equation}\label{eq:reconstructionError}
\begin{split}
&\|\mathbf{X} - \mathbf{X}(\mathbf{D\Phi}\mathbf{ S}^{\frac{1}{2}})(\mathbf{D\Phi}\mathbf{ S}^{\frac{1}{2}})^{T}||_F^2 \\
&= \|\mathbf{X} - \mathbf{XD\Phi}\mathbf{S}\mathbf{\Phi}^{T}\mathbf{D}^T||_F^2
\end{split}
\end{equation}
where $\|\mathbf{A}\|_F$ denotes the Frobenius norm of matrix $\mathbf{A}$ which is defined as the square root of the sum of the squares of every elements in $\mathbf{A}$. Note that $\|\mathbf{A}\|_F = (trace(\mathbf{A}^T\mathbf{A}))^{1/2} = (trace(\mathbf{A}\mathbf{A})^T)^{1/2}$ where the $trace()$ operator returns the sum of diagonal entries of an input matrix. As shown in the following Proposition \ref{proposition:equivalance}, the reconstruction error of count-sketch as shown in (\ref{eq:reconstructionError}) is equivalent to the objective function of applying $k$-means clustering to cluster the $d$ columns of $\mathbf{XD}$ into $r$ clusters.

\begin{proposition}\label{proposition:equivalance}
The reconstruction error of count-sketch 
$\|\mathbf{X} - \mathbf{XD\Phi S}\mathbf{\Phi}^{T}\mathbf{D}^T||_F^2$ is equivalent to the objective function of $k$-means clustering on the columns of matrix product $\mathbf{M} = \mathbf{XD}$ if we treat $\mathbf{\Phi}$ as a learnable variable which denotes the cluster membership of each column in $\mathbf{M}$.
\end{proposition}

\begin{proof}
We first rewrite the reconstruction error as $\|\mathbf{X} - \mathbf{XD\Phi S}\mathbf{\Phi}^{T}\mathbf{D}^T||_F^2 = \|\mathbf{X} - \mathbf{XD}\mathbf{D}^T + \mathbf{XD}\mathbf{D}^T - \mathbf{XD\Phi S}\mathbf{\Phi}^{T}\mathbf{D}^T||_F^2$. Note that $\mathbf{D}$ is a $d \times d$ diagonal matrix with each diagonal entry either 1 or $-1$, therefore $\mathbf{X} = \mathbf{XD}\mathbf{D}^T$. Let us use $\mathbf{M}$ to denote $\mathbf{XD}$, we will have 
\begin{equation}\label{eq:XDD_t}
\begin{split}
&\|\mathbf{X} - \mathbf{XD\Phi S}\mathbf{\Phi}^{T}\mathbf{D}^T||_F^2  \\ 
=&\|\mathbf{XD}\mathbf{D}^T - \mathbf{XD\Phi S}\mathbf{\Phi}^{T}\mathbf{D}^T||_F^2  \\
 = &\| \mathbf{M}\mathbf{D}^T - \mathbf{M\Phi S}\mathbf{\Phi}^{T}\mathbf{D}^T||_F^2 
\end{split}
\end{equation}

Next, we will show that $\|\mathbf{M}\mathbf{D}^T - \mathbf{M\Phi S}\mathbf{\Phi}^{T}\mathbf{D}^T||_F^2 = \|\mathbf{M} - \mathbf{M\Phi S}\mathbf{\Phi}^{T}||_F^2$ as follows, 
\begin{equation}\label{eq:XD}
\begin{split}
&\ \ \ \ \|\mathbf{M}\mathbf{D}^T - \mathbf{M\Phi S}\mathbf{\Phi}^{T}\mathbf{D}^T\|_F^2 \\
& = trace((\mathbf{M}\mathbf{D}^T - \mathbf{M\Phi S}\mathbf{\Phi}^{T}\mathbf{D}^T)\\
&\ \ \ \ \ \ \ \ \ \ \ \ \ \ \ \ \ \ \ \ \ \ \ \ \ \ \ \ \ \ \ \ \ \ \ \ \ \ \ \  (\mathbf{M}\mathbf{D}^T - \mathbf{M\Phi S}\mathbf{\Phi}^{T}\mathbf{D}^T)^T) 
\\
& = trace((\mathbf{M} - \mathbf{M\Phi S}\mathbf{\Phi}^{T})\mathbf{D}^T\mathbf{D}(\mathbf{M} - \mathbf{M\Phi S}\mathbf{\Phi}^{T})^T) 
\\
& = trace((\mathbf{M} - \mathbf{M\Phi S}\mathbf{\Phi}^{T})(\mathbf{M} - \mathbf{M\Phi S}\mathbf{\Phi}^{T})^T). \\
& = \|\mathbf{M} - \mathbf{M\Phi S}\mathbf{\Phi}^{T}||_F^2. 
\end{split}
\end{equation}

Combining (\ref{eq:XDD_t}) and (\ref{eq:XD}), the reconstruction error of count-sketch can be rewritten as
\begin{equation}\label{eq:M_clustering}
\begin{split}
& \|\mathbf{X} - \mathbf{XD\Phi S}\mathbf{\Phi}^{T}\mathbf{D}^T||_F^2 = \|\mathbf{M} - \mathbf{M\Phi S}\mathbf{\Phi}^{T}||_F^2 \\
\end{split}
\end{equation}

Based on the definition of matrix $\mathbf{\Phi}$, $\mathbf{\Phi}$ is a $d \times r$ indicator matrix which each row has only one non-zero entry. Therefore, $\mathbf{\Phi}$ can viewed as a cluster membership indicator matrix which corresponds to randomly assign $d$ columns of matrix $\mathbf{M}$ into $r$ clusters. The non-zero element $\mathbf{\Phi}_{ij} = 1$ in $i$-th row of $\mathbf{\Phi}$ denotes the $i$-th column in $\mathbf{M}$ is assigned to cluster $j$. Note that the $i$-th column of matrix product $\mathbf{M\Phi S}\mathbf{\Phi}^{T}$ is the centroid of the cluster where the $i$-th column $\mathbf{M}_{(:,i)}$ belongs to. Therefore
\begin{equation}\label{eq:M_clustering2}
\|\mathbf{M} - \mathbf{M\Phi S}\mathbf{\Phi}^{T}||_F^2 = \sum_{i=1}^{d}\|\mathbf{M}_{(:,i)}-\mathbf{c}_{I(\mathbf{M}_{(:,i)})}\|_2^2, 
\end{equation}
where $I(\mathbf{M}_{(:,i)})$ returns the index of the cluster that the $i$-th column $\mathbf{M}_{(:,i)}$ belongs to and  $\mathbf{c}_{I(\mathbf{M}_{(:,i)})}$ is the centroid of that cluster. By treating $\mathbf{\Phi}$ as a learnable variable which denotes the cluster membership, the reconstruction error of count-sketch is the same as the objective function of $k$-means algorithm on the columns of $\mathbf{M}$ as shown in \ref{eq:M_clustering2}. 
    
\end{proof}
Our proposition \ref{proposition:equivalance} provides an interesting connection between count-Sketch and $k$-means clustering. In the count-Sketch algorithm, the clustering membership indicator matrix $\mathbf{\Phi}$ is randomly generated which does not consider intrinsic data properties and it could result in bad embedding with high reconstruction error. 



\begin{algorithm}[tb]
\caption{$\epsilon$-$\mathcal{L}_1$ ball projection \cite{sculley2010web}}
\begin{algorithmic}
\STATE \textbf{Input}: $\mathbf{c} \in \mathbb{R}^{n}$, $\mathcal{L}_1$ ball radius $\lambda$, tolerance parameter $\epsilon$
\STATE \textbf{Output}: projected sparse vector $\mathbf{c} \in \mathbb{R}^{n}$ 
\end{algorithmic}
	\begin{algorithmic}[1]
		\STATE \algorithmicif{ $\|\mathbf{c}\| \le \lambda(1 + \epsilon)$ } \algorithmicreturn{ $\mathbf{c}$ }
		\STATE $l$ = 0; $u$ = $\|\mathbf{c}\|_{\infty}$; $r$ = $\|\mathbf{c}\|_1$  
		\WHILE[Bisection to find $\theta$]{$r>\lambda(1+\epsilon)$ \OR $r < \lambda$}
		 \STATE $\theta = \frac{l+u}{2}$
		 \STATE $r = \sum_{i=1}^{n} \text{max}(0, |c_i| - \theta)$
		 \STATE \algorithmicif{ $r < \lambda$ } \algorithmicthen{ $u = \theta$ } \algorithmicelse{ $l = \theta$ }
		 \ENDWHILE
		\FOR[$\mathcal{L}_1$ ball projection]{\text{$i = 1$ to $n$}} 
		\STATE ${c}_{i} = \text{sign}({c}_{i})\text{max}(0, \left|c_{i}\right|-\theta)$ 
		\ENDFOR
	\end{algorithmic}\label{alg:l1_projection}
    
\end{algorithm}

\subsection{Improved count-sketch by $k$-means and $\mathcal{L}_1$ ball projection}
As shown in (\ref{eq:M_clustering2}), the reconstruction error of count-sketch can be improved by replacing the random cluster membership indicator matrix $\mathbf{\Phi}$ in the original count-sketch algorithm by a cluster membership indicator matrix produced by $k$-means algorithm on the columns of $\mathbf{M}$. Motivated by this observation, we propose to use $k$-means algorithm to learn the cluster membership indicator matrix $\mathbf{\Phi}$ from data for lower reconstruction error. Therefore, the new cluster centers returned by $k$-means with $k$ = $r$, which equals to $\mathbf{XD\Phi S}$, can be used as the new low-dimensional feature representation. And this new method will result in low reconstruction error than the original count-sketch method. 

Apart from the reconstruction error, as mentioned earlier, another limitation of count-sketch is that it may not preserve the sparsity rate of the input data $\mathbf{X}$. In other words, the new data presentation $\mathbf{\widetilde{X}}$ could be dense even if the original data $\mathbf{X}$ is highly sparse data. This limitation could make the subsequent algorithm on projected data $\mathbf{\widetilde{X}}$ be even slower than just using the original data $\mathbf{X}$ without count-sketch. Therefore, instead of using the Lloyd’s classic $k$-means algorithm \cite{lloyd1982least}, we would like to develop a new method to obtain very sparse cluster centers. We propose to obtain sparse cluster centers by optimizing the objective of $k$-means as shown in (\ref{eq:M_clustering2}) using gradient descent together with $\mathcal{L}_1$ ball projection \cite{duchi2008efficient} in each update. 

The gradient of the $k$-means objective function $\sum_{i=1}^{d}\|\mathbf{M}_{(:,i)}-\mathbf{c}_{I(\mathbf{M}_{(:,i)})}\|_2^2$ with respect to the $j$-th cluster center $\mathbf{c}_j$ is

\begin{equation}\label{eq:k_means_gradient}
\nabla \mathbf{c}_j= \sum_{i=1}^{d}-2\delta(I(\mathbf{M}_{(:,i)}), j)(\mathbf{M}_{(:,i)} - \mathbf{c}_j),
\end{equation}

where $\delta(I(\mathbf{M}_{(:,i)}), j)$ is a binary function which return 1 if $I(\mathbf{M}_{(:,i)})$ equals to $j$ (i.e., the $i$-th column of $\mathbf{M}$ belongs to the $j$-th cluster), otherwise it returns 0. In other word, the computation of gradient $\nabla \mathbf{c}_j$ only depends on columns that belongs to $j$-th cluster in current iteration. 

By using gradient descent, in each iteration, the cluster center $\mathbf{c}_j$ can be updated as  

\begin{equation}\label{eq:cj_update}
\mathbf{c}_j = \mathbf{c}_j - \eta \nabla \mathbf{c}_j,
\end{equation}

where $\eta$ is the learning rate. However, directly using (\ref{eq:cj_update}) will not produce sparse cluster centers. 

To obtain sparse cluster centers, we will use $\epsilon$-$\mathcal{L}_1$ ball projection to make $\mathbf{c}_j$ be a sparse vector. The $\epsilon$-$\mathcal{L}_1$ ball projection is proposed in \cite{sculley2010web} which is approximated extension of exact $\mathcal{L}_1$ ball projection \cite{duchi2008efficient}. $\epsilon$-$\mathcal{L}_1$ is very effective at getting sparse cluster centers as shown in \cite{sculley2010web}. The basic idea of $\epsilon$-$\mathcal{L}_1$ ball projection is to use bisection to find a value $\theta$ that projects a dense vector $\mathbf{c}_j$ to an $\mathcal{L}_1$ ball with radius between $\lambda$ and $(1 + \epsilon)\lambda$. After $\theta$ is found, $\epsilon$-$\mathcal{L}_1$ ball projection will map the $i$-th entry in $\mathbf{c}_j$ (denoted as $c_{ji}$) to 
\begin{equation}\label{eq:cj_projection}
{c}_{ji} = \text{sign}({c}_{ji})\text{max}(0, \left|c_{ji}\right|-\theta).
\end{equation}

As shown in (\ref{eq:cj_projection}), the resulting cluster centers $\mathbf{c}_j$s will be sparse vectors since $\text{max}(0, \left|c_{ji}\right|-\theta)$ will make an element to 0 if its absolute value is smaller than $\theta$. The whole procedure of $\epsilon$-$\mathcal{L}_1$ ball projection is described in Algorithm \ref{alg:l1_projection}. 

By using Algorithm \ref{alg:l1_projection} in each iteration of optimizing $k$-means objective function by gradient descent, we will get sparse cluster centers.  

\begin{algorithm}[tb]
\caption{\textbf{E}ffective and \textbf{S}parse \textbf{C}ount-S\textbf{K}etch (ESCK)}
\begin{algorithmic}
\STATE \textbf{Input}: $\textbf{X} \in \mathbf{R}^{n \times d}$, reduced dimension $r$, iteration $t$, parameter $\epsilon$, $\lambda$ for $\mathcal{L}_1$ ball projection;
\STATE \textbf{Output}: low-dimensional data representation $\mathbf{\widetilde{X}} \in \mathbb{R}^{n \ \times r}$ and the learnt cluster membership indicator matrix $\mathbf{\Phi}$
\end{algorithmic}
	\begin{algorithmic}[1]
		\STATE Generate a diagonal random sign matrix $\mathbf{D}$ 
		\STATE Compute $\mathbf{M} = \mathbf{X}\mathbf{D}$
		\STATE Randomly pick $r$ columns from $\mathbf{M}$ as the cluster centers $\{\mathbf{c}_j\}_{j=1}^{r}$ 
		\FOR{\text{$iter = 1$ to $t$}}
		\STATE Create all zero matrix $\mathbf{\Phi} \in \mathbb{R}^{d \times r}$
		\FOR{\text{$i = 1$ to $d$}}
		\STATE $j = \text{argmin}_j \|\mathbf{M}_{(:,i)} - \mathbf{c}_j\|_2^2$
		\STATE $\mathbf{\Phi}_{i,j} = 1$
		\ENDFOR
		\STATE Update each cluster centers using (\ref{eq:cj_update})
		\STATE Obtain sparse cluster centers using Algorithm 1 
		\ENDFOR
	\STATE \algorithmicreturn{ $\mathbf{\widetilde{X}} = [\mathbf{c}_1, \mathbf{c}_2, \dots, \mathbf{c}_r]$ and $\mathbf{\Phi}$ }
	\end{algorithmic}
    
\end{algorithm}\label{alg:esCountSketch}

\subsection{Algorithm Implementation and Analysis}
We summarize our proposed algorithm for improving the original count-sketch in algorithm \ref{alg:esCountSketch} and named it as \textbf{E}ffective and \textbf{S}parse \textbf{C}ount-S\textbf{K}etch (ESCK). Our proposed algorithm first obtain $\mathbf{M}$ as shown in step 1-2 which is the same as the original count-sketch. The contribution of our proposed algorithm is to replace the randomly generated cluster membership indicator matrix $\Phi$ in count-sketch with the learned cluster membership indicator matrix $\Phi$. The $r$ sparse cluster centers will be used as the low-dimensional data representation. 
As shown from step 3 to 12, sparse cluster centers are obtained by using gradient descent with $\epsilon$-$\mathcal{L}_1$ ball projection to cluster $d$ columns into $r$ groups.

With respect to time complexity, step 2 only needs $O(nnz(\mathbf{X}))$ time because of $\mathbf{D}$ is a diagonal matrix. The time complexity for Step 3 to 12 is upper bounded by $O(ndrt)$ where $t$ is the number of iterations. For sparse input data, the time complexity in each iteration for updating cluster centers will smaller than $O(ndr)$ since both data and clusters are sparse. Empirically, the $k$-means algorithm using gradient descent converges very fast and only a few iterations is needed. In our experiments, we will show that our proposed method is only several time slower than count-sketch but the classification accuracy obtained by our method is much larger than count-sketch and other methods. Note that our proposed method can also return the learned cluster membership indicator matrix $\mathbf{\Phi}$, therefore, our algorithm can also be extended to inductive setting and generate the feature mapping for new unseen data by using $\mathbf{\widetilde{X}} = \mathbf{XD\Phi S}$ which enjoys the same low computational complexity as the original count-Sketch.

\begin{table}[htb]
	\centering
	\caption{Summary of Experimental Datasets}
	\begin{tabular}{ccccc}
		\hline
		\multirow{2}{*}{Dataset}& \# of  & \# of & \# of & sparsity      \\
		        &  samples     &  features    & classes     &  rate           \\
		\hline
		{ usps}	  & 9,298	 &  256 & 10 & 0\%  \\
	    { mnist}    & 60,000   &  780 & 10 & 87.78\%\\
		{ gisette}   & 7,000  &  5,000  & 2 & 0.85\%\\
		{ real-sim}  & 72,309 &  20,958 & 2 & 99.75\%\\
		{ rcv1-binary}      & 20,242   &  47,236  & 2 & 99.84\% \\
		{ rcv1-multi}    & 15,564  &  47,236 & 53 & 99.86\%\\
	
		\hline
		\label{table.data2}
	\end{tabular}	
\end{table} 

\begin{table*}[htb]
  \centering
  \caption{Experimental Results of Different Random Matrix Sketching Methods}
    \resizebox{\textwidth}{!}{ %
    \begin{tabular}{c|c|c|c|c|c|c|c}
\hline
    \multirow{2}[2]{*}{} & \multirow{2}[2]{*}{Performance} & usps & mnist & gisette & \multicolumn{1}{c|}{real-sim} & \multicolumn{1}{c|}{rcv1-binary} & rcv1-multi \\
          & \multicolumn{1}{c|}{} &($r$=30) & ($r$=100) & ($r$=256) & \multicolumn{1}{c|}{(r=256)} & \multicolumn{1}{c|}{($r$=256)} & ($r$=256) \\
\hline
\hline
    \multicolumn{1}{c|}{\multirow{3}[2]{*}{Gaussian}} & Accuracy(\%) & 90.60 $\pm$ 0.01 & \multicolumn{1}{c|}{88.92 $\pm$ 0.03} & 90.70 $\pm$ 0.01 & \multicolumn{1}{c|}{79.25 $\pm$ 0.06} & \multicolumn{1}{c|}{81.63 $\pm$ 0.01} & 69.71 $\pm$ 0.01 \\
          & Sparsity rate & 0\% & 0\% & 0\% & \multicolumn{1}{c|}{ 0\% } & \multicolumn{1}{c|}{0\%} &  0\% \\
          & Prediction time(ms) & 2.89ms & 20.84ms & 3.01ms & \multicolumn{1}{c|}{31.15ms} & \multicolumn{1}{c|}{7.95ms} & 22.42ms \\\hline

    \multicolumn{1}{c|}{\multirow{3}[1]{*}{Achlioptas}} & Accuracy(\%) & 90.17$\pm$0.03 & \multicolumn{1}{c|}{87.70$\pm$0.02} & 89.80$\pm$ 0.03 & \multicolumn{1}{c|}{76.85$\pm$0.06} & \multicolumn{1}{c|}{81.86$\pm$0.01} & 67.56$\pm$0.07 \\
    \multicolumn{1}{c|}{} & Sparsity rate & 0\% & 0.01\% & 0\% & \multicolumn{1}{c|}{1.07\%} & \multicolumn{1}{c|}{0.03\%} & 0.03\% \\
          & Prediction time(ms) &\multicolumn{1}{c|}{3.12ms} & 54.98ms & 3.98ms & \multicolumn{1}{c|}{41.92ms} & \multicolumn{1}{c|}{11.71ms} & 88.76ms \\
\hline

  \multicolumn{1}{c|}{\multirow{3}[1]{*}{Count-Sketch}} & Accuracy(\%) & 90.74 $\pm$ 0.01 & \multicolumn{1}{c|}{87.66 $\pm$ 0.01}& 90.37 $\pm$ 0.02 & \multicolumn{1}{c|}{77.21 $\pm$ 0.06} & \multicolumn{1}{c|}{80.19$\pm$0.02} & 69.38$\pm$ 0.06 \\
          & Sparsity rate & 0\%  & 1.72\% & 0\% & \multicolumn{1}{c|}{73.36\%} & \multicolumn{1}{c|}{73.65\%} & 74.47\% \\
          & Prediction time(ms) & 2.98ms & 50.86ms & 4.46ms & \multicolumn{1}{c|}{11.96ms} & \multicolumn{1}{c|}{4.46ms} & 24.93ms \\\hline

\multirow{3}{*}{SRHT}&Accuracy(\%) &89.86 $\pm$ 1.66&87.14 $\pm$ 0.84&90.45 $\pm$ 0.87&78.37 $\pm$ 0.20&80.29 $\pm$ 0.69&68.50 $\pm$ 0.29\\
&Sparsity rate& 0\% & 0\% & 0\% & 0\% & 0\% &0\%\\
&Prediction time(ms)&3.3ms&79.35ms&3.6ms&22.6ms&6.96ms&103ms\\
\hline
\hline

\multirow{3}{*}{SRHT-topr}&Accuracy(\%) &90.68 $\pm$ 1.51&88.15 $\pm$ 0.77&92.45 $\pm$ 0.55&82.48 $\pm$ 0.19&82.14 $\pm$ 0.24&71.01 $\pm$ 0.77\\
&Sparsity rate& 0\% & 0\% & 0\% & 0\% & 0\% &0\%\\
&Prediction time(ms)&4.22ms&75.23ms&3.3ms&46.6ms&6.14ms&101ms\\

\hline
    \multicolumn{1}{c|}{\multirow{3}[2]{*}{ESCK-full}} & Accuracy(\%) & \textit{90.90$\pm$ 0.02} & \multicolumn{1}{c|}{\textbf{90.60$\pm$0.02}} & \textit{93.25$\pm$ 0.02}& \multicolumn{1}{c|}{\textbf{88.68$\pm$0.07}} & \multicolumn{1}{c|}{\textbf{92.91$\pm$0.01}} & \textbf{78.99 $\pm$ 0.01} \\
          & sparsity rate & 50.81\% &    43.10\%  & 66.09\%  & \multicolumn{1}{c|}{89.57\%} & \multicolumn{1}{c|}{87.61\%} & 88.44\% \\
          & Prediction time(ms) & 0.97ms &    15.81ms & 0.99ms & \multicolumn{1}{c|}{3.99ms} & \multicolumn{1}{c|}{1.01ms} & 13.51ms \\
\hline
    \multicolumn{1}{c|}{\multirow{3}[2]{*}{ESCK-miniBatch}} & Accuracy(\%) & \textbf{91.90$\pm$ 0.01} & \multicolumn{1}{c|}{\textit{90.50$\pm$0.02}} & \textbf{94.45 $\pm$ 0.03}& \multicolumn{1}{c|}{\textit{88.25$\pm$0.08}} & \multicolumn{1}{c|}{\textit{90.01$\pm$0.02}} & \textit{77.13 $\pm$ 0.01} \\
          & sparsity rate & 46.78\%  &    40.29\%  &  37.58\%  & \multicolumn{1}{c|}{97.47\%} & \multicolumn{1}{c|}{94.78\%} & 95.37\% \\
          & Prediction time(ms) & 1.21ms & 16.86ms & 2.26ms & \multicolumn{1}{c|}{3.28ms} & \multicolumn{1}{c|}{0.99ms} & 5.98ms \\
\hline
    
    \end{tabular}}%
  \label{table.data3}%
\end{table*}%

\begin{table*}[h]
\centering
	\caption{Comparison of the embedding during the training stage}
\begin{tabular}{|l|c|l|c|c|c|c|}
\hline
\multirow{2}{*}{} & \multicolumn{6}{c|}{embedding time during training stage}                                                                                                                                                         \\ \cline{2-7} 
                  & \multicolumn{1}{l|}{usps} & mnist                      & \multicolumn{1}{l|}{gisette} & \multicolumn{1}{l|}{real-sim} & \multicolumn{1}{l|}{rcv1-bianry} & \multicolumn{1}{l|}{rcv1-multi} \\ \hline
Count-Sketch      & 1.2s                      & \multicolumn{1}{c|}{5.1s}  & 2.6s                         & 2.6s                          & 1.9s                             & 1.1s                            \\ \hline
ESCK-full         & 5.3s                      & \multicolumn{1}{c|}{26.4s} & 10.2s                        & 38.5s                         & 12.6s                            & 8.5s                            \\ \hline
ESCK-miniBatch    & \multicolumn{1}{l|}{1.1s} &  6.5s                       & 4.5s                         & 13.5s                         & 5.3s                             & 4.1s                            \\ \hline
\end{tabular}
\label{table:TrainingTime}
\end{table*}

\section{Experiments}
In this section, we compared our proposed algorithm with 
several different commonly-used random dimensionality reduction algorithms based on six real-life datasets. These six datasets are downloaded from LIBSVM website\cite{chang2011libsvm}. The summarization of these six datasets is shown in Table \ref{table.data2}. The sparsity rate as shown in the last column is the fraction of zeros in each input data matrix $\mathbf{X}$. As shown in Table \ref{table.data2}, there are four sparse datasets (mnist, real-sim, rcv1-binary, rcv1-multi) and two dense datasets (usps, gisette).

We evaluate the performance of following seven matrix sketching methods: 
\begin{itemize}
\item Gaussian: The sketching matrix is a random Gaussian Matrix \cite{dasgupta1999elementary}
\item Achlioptas: The sketching matrix is randomly generated from a discrete distribution and is sparser than Gaussian matrix \cite{achlioptas2003database}. 
\item Count-Sketch: original oblivious count-sketch method \cite{clarkson2017low}
\item SRHT : The sketching matrix is generated by the Subsampled Randomized Hadamard Transform (SRHT) \cite{tropp2011improved}
\item SRHT-topr  An improved variant of SRHT which is data-dependent  
\item ESCK-full: our proposed method that use full batch gradient descent with $\epsilon$-$\mathcal{L}_1$ ball projection to get the $k$-means centers.

\item ESCK-miniBatch: our proposed method that uses mini-batch gradient descent with $\epsilon$-$\mathcal{L}_1$ ball projection to get the $k$-means centers. This is more efficient than ESCK-full but with slightly lower accuracy. 
\end{itemize}
\textbf{Experimental Setting}. For the two dense datasets (usps and gisette), the feature values are scaled to [$-1$,1] using min-max normalization. We use five-fold cross validation to evaluate the accuracy. The regularization parameter $C$ in SVM  is chosen from $\{10^{-5},10^{-4}, \dots, 10^{4},10^{5}\}$. The $\epsilon$ parameter for $\epsilon$-$\mathcal{L}_1$ ball projection is fixed to $0.1$ and $\lambda$ parameter is chosen from $\{10, 20, 30, 40\}$ . Our experiments are performed on a desktop with Intel(R) Core(TM) i7-9700 CPU and @ 3.00GHz and 16.0 GB RAM.   

\textbf{Experimental Results}. We report the classification accuracy, sparsity rate of the sketched matrix and prediction time of different algorithms in Table \ref{table.data3}. The projected dimension $r$ for each dataset is given in the first row of this table. The results for different settings of projected dimension $r$ will be discussed later. The first four methods are data-oblivious random projection methods and the last three are data-dependent random projection methods. The best accuracy for each dataset is in bold and the second best accuracy for each dataset is in italic.

As shown in this Table, the data-dependent matrix sketching methods (i.e., SRHT-topr, ESCK-full and ESCK-miniBatch) get higher accuracy than data-independent matrix sketching methods. Among these six datasets, the proposed ESCK-full algorithm achieves the best accuracy on four datasets (i.e., mnist, real-sim, rcv1-binary, rcv1-multi) and the second best accuracy in two datasets (i.e., usps and gisette). Overall, our proposed method ESCK-full gets the best accuracy. The proposed method ESCK-miniBatch gets slightly lower accuracy than ESCK-full but gets higher accuracy than the other five matrix sketching methods. The results in Table \ref{table.data3} demonstrate that our proposed methods achieve better accuracy than other methods.

\begin{table}[htb]
		\centering
		\caption{Prediction Costs for Different Algorithms on a Single Input Sample $\mathbf{x}$}
		\begin{tabular}{|c|c|c|}
		\hline
        \multirow{2}{*}{Algorithms} & \multicolumn{2}{c|}{Prediction Cost} \\
        \cline{2-3}
         & Projection Cost & Classification Cost\\
        \hline
        \hline
	    Gaussian  & $O(dr)$ & $O(r)$  \\
        \hline
        Achlioptas &  $O(dr)$ & $O(r)$\\
        \hline
        Count-Sketch & $O(nnz(\mathbf{x}))$ & $O(nnz(\Tilde{\mathbf{x}}))$  \\
        \hline
        SRHT   & $O(dlog(d))$ & $O(r)$  \\
        \hline
        SRHT-topr & $O(dlog(d))$ &  $O(r)$  \\
		\hline
		ESCK & $O(nnz(\mathbf{x}))$ & $O(nnz(\Tilde{\mathbf{x}}))$  \\
		\hline
		
	\end{tabular}
    \label{table:cost_different_algorithms}
\end{table}

\begin{figure*}[h]
\centering
\subfigure{
\includegraphics[scale=0.4]{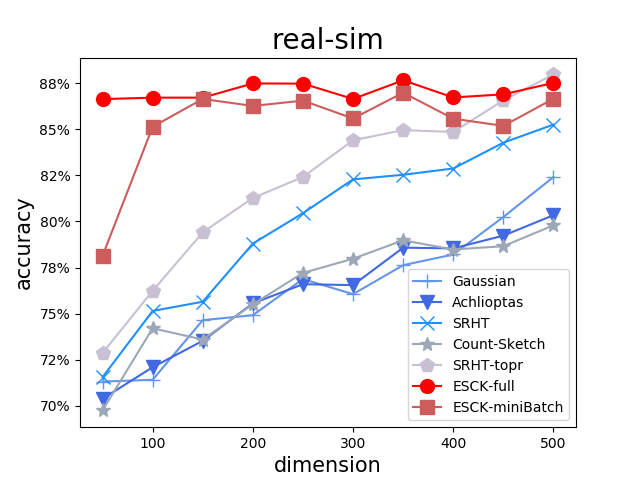}
}
\subfigure{
\includegraphics[scale=0.4]{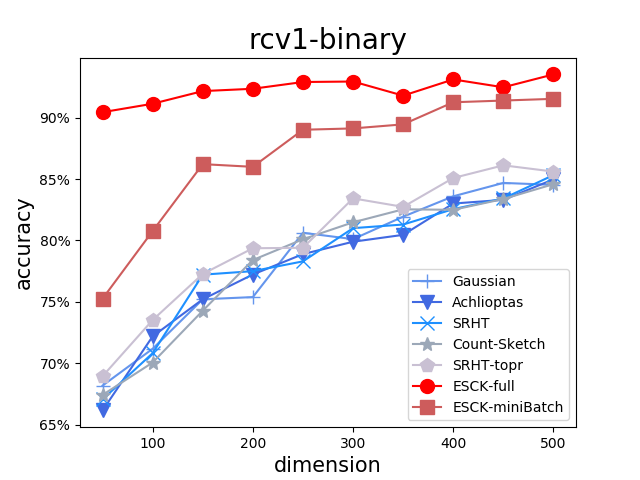}
}\\

\subfigure{
\includegraphics[scale=0.4]{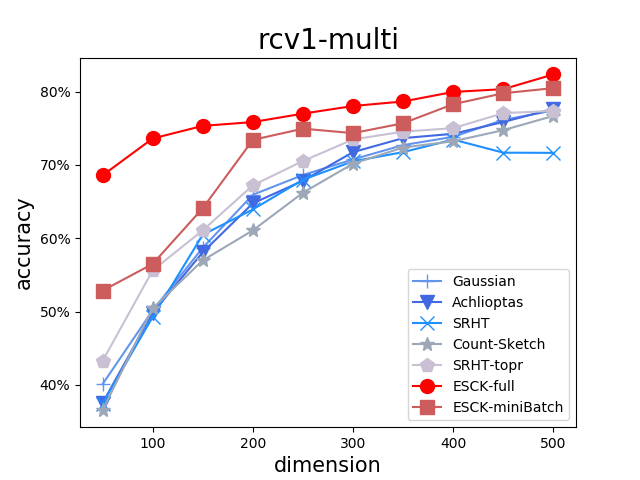}
}
\subfigure{
\includegraphics[scale=0.4]{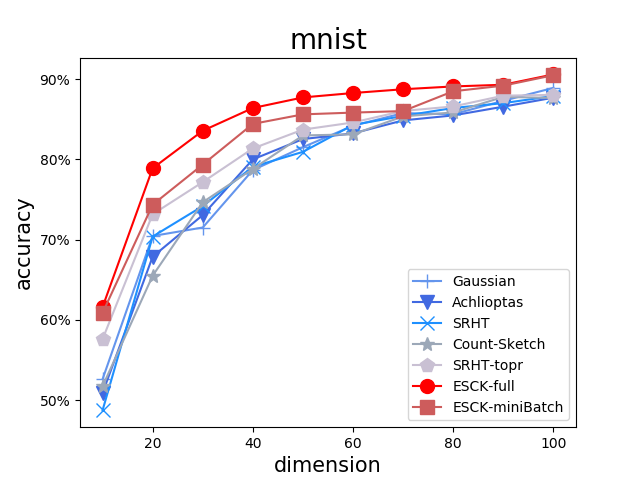}
\label{fig.5}
}
\caption{Impact of Projection Dimension $r$}\label{fig:impact_r}
\end{figure*}

\begin{figure*}[!h]
\centering
\subfigure[real-sim]{
\includegraphics[width=2.25in]{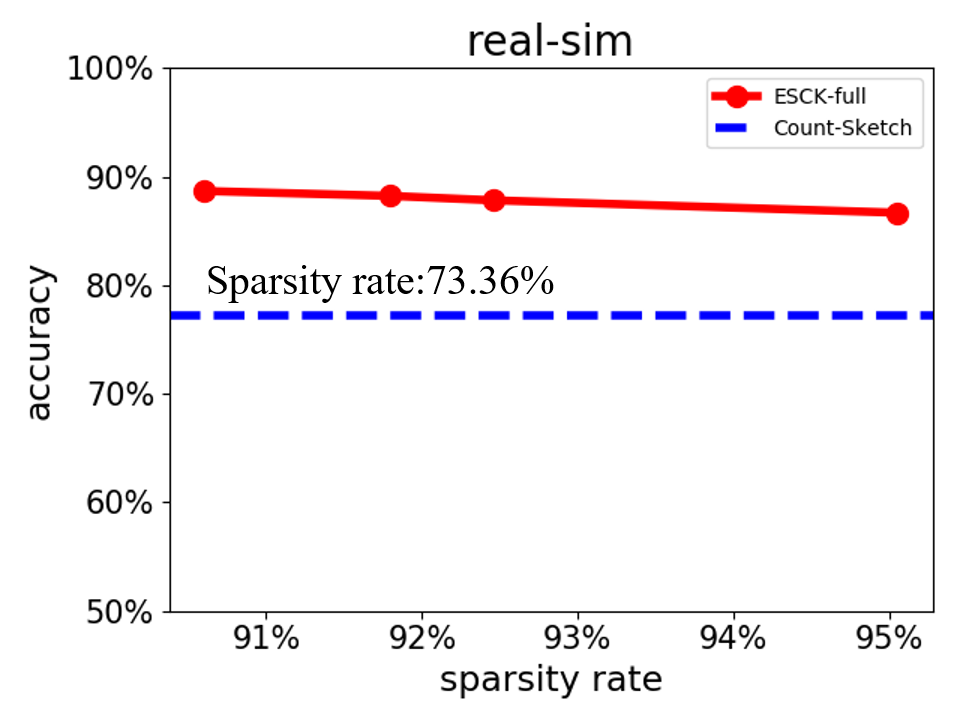}}
\subfigure[rcv1-binary]{
\includegraphics[width=2.25in]{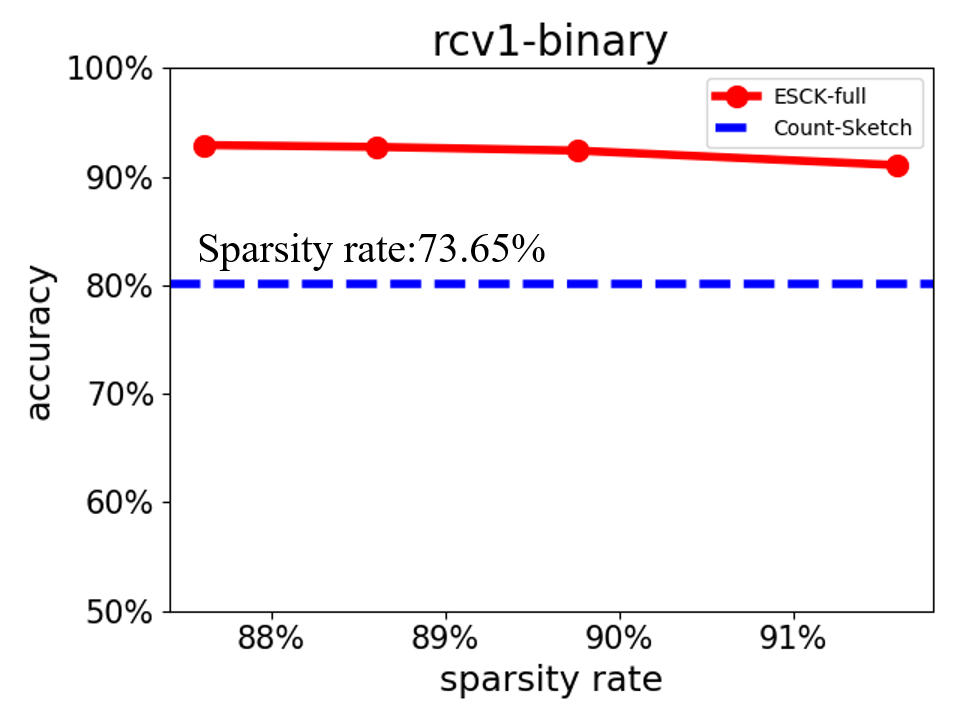}}
\subfigure[mnist]{
\includegraphics[width=2.25in]{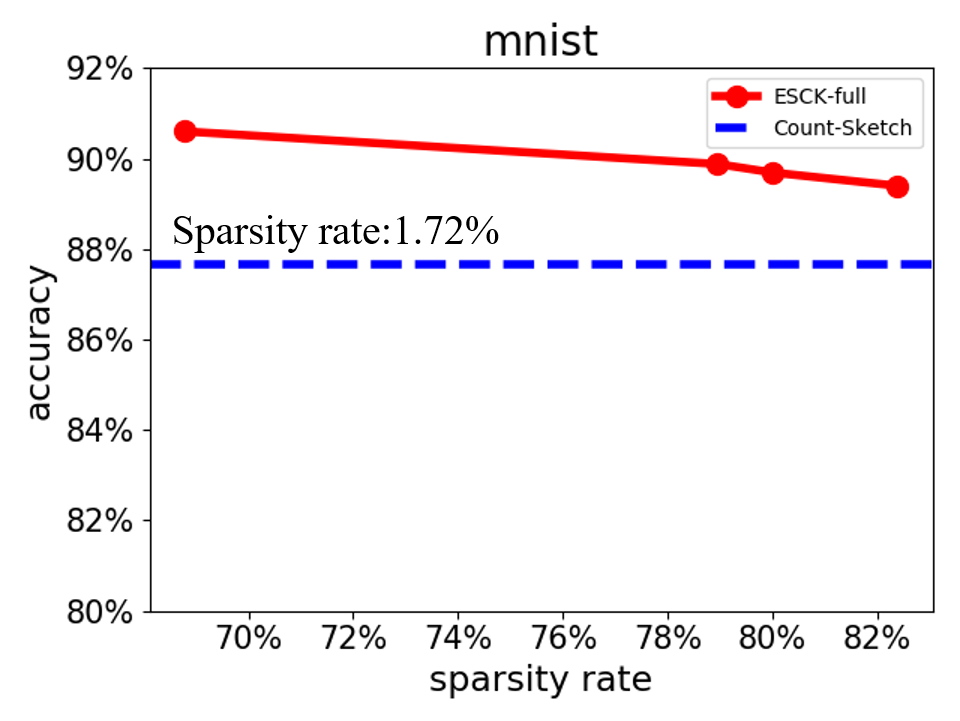}}
\centering
\caption{Accuracies with Different Sparsity Rates}
 \label{fig.contrl}
\end{figure*}

With respect to the sparsity rate of the sketched data, as expected, Gaussian, Achlioptas, SRHT and SRHT-topr will produce dense data even if the input data is sparse. The original count-sketch method and our proposed methods can produce sparse embedding for highly sparse input data. The sparsity rate of the sketched data produced by our proposed methods is higher than the count-sketch. Furthermore, our proposed method could result in sparse embedding for dense input data (e.g., usps and gisette). With respect to the prediction time, the prediction time of our methods is lower than other methods. The prediction cost for different algorithms is summarized in Table \ref{table:cost_different_algorithms}. Both count-sketch and our proposed ESCK are very efficient for prediction. 

We also compare the embedding time of our proposed method with the original count-sketch during the training stage. The results are shown in Table \ref{table:TrainingTime}. As expected, our proposed methods will be several times slower than the original count-sketch since we need to perform $k$-means clustering on the columns of $\mathbf{M}$. ESCK-miniBatch is faster than ESCK-full. 

\textbf{Impact of Projection Dimension $r$.} In Figure \ref{fig:impact_r}, we show the experimental results of all algorithms with different projection dimension $r$. As shown in this figure, our proposed method ESCK-full consistently get better accuracy than other matrix sketching methods. The other two data-dependent matrix sketching methods ESCK-minibatch and SRHT-topr also gets better than the four data-oblivious matrix sketching method. When the parameter $r$ is small, the accuracy improvement of our proposed method is large on real-sim, rcv1-binary and rcv1-multi datasets. 

\textbf{Impact of Sparse Sketched Matrix}
By tuning the $\lambda$ parameter $\epsilon$-$\mathcal{L}_1$ ball projection, our proposed method can result in a very sparse sketched matrix $\Tilde{\mathbf{X}}$. In this section, we would like to explore how the sparsity rate of the sketched matrix affects the classification accuracy. In Figure \ref{fig.contrl}, we show the sparsity rate and accuracy for count-sketch and ESCK-full. The blue dashed line shows the accuracy of count-sketch and the sparsity rate is annotated by the text above this line. The red line shows the accuracies of ESCK-full with different sparsity rate of the sketched matrix. 
As shown in Figure \ref{fig.contrl}, our proposed methods obtain better accuracy than count-sketch with a higher sparsity rate. As the sparsity rate increased, we can observe that accuracy could slightly decrease but still higher than count-sketch. On the mnist dataset, the count-sketch method generates a dense sketched matrix with a sparsity rate equals to 1.72\% and the accuracy of the subsequent classifier is 87.65\%. In comparison, the ESCK-full can generate a sparse sketched matrix with higher classification accuracy.

\section{Conclusion}
In this paper, we propose a novel data-dependent count-sketch algorithm that can produce more effective and sparse subspace embedding than the original data-independent count-sketch algorithm. Our new method applies $k$-means clustering algorithm to obtain the sketched data matrix. Sparse sketched data matrix is obtained by using gradient descent with $\epsilon$-$\mathcal{L}_1$ ball projection to optimize the $k$-means clustering objective function. We compared our proposed algorithm with the other five matrix sketching algorithms. Our experimental results on six real-life datasets have demonstrated that our proposed methods achieve higher classification accuracies than count-sketch and other matrix sketching methods. Also, our proposed methods can produce sketched matrix with high sparsity rate than other methods that can make the subsequent classification model more efficient than other methods.   


\begin{small}
\bibliography{myRef}
\end{small}

\end{document}